\documentclass{article}
\usepackage{ijcai21}

\usepackage{times}

\usepackage{soul}
\usepackage{url}
\usepackage[hidelinks]{hyperref}
\usepackage[utf8]{inputenc}
\usepackage[small]{caption}
\usepackage{graphicx}
\usepackage{amsmath}
\usepackage{booktabs}
\usepackage[utf8]{inputenc}
\usepackage{xcolor}
\usepackage{booktabs}
\usepackage{amsthm}
\usepackage{thmtools,thm-restate}
\usepackage{listings}
\usepackage{inconsolata}
\usepackage{tikz}
\usepackage{pgfplots}
\usepackage{amsmath}
\usepackage{microtype}
\usepackage{standalone}
\usepackage{xr}

\usepackage{pgfkeys}
    \newenvironment{customlegend}[1][]{%
        \begingroup
        \csname pgfplots@init@cleared@structures\endcsname
        \pgfplotsset{#1}%
    }{%
        \csname pgfplots@createlegend\endcsname
        \endgroup
    }%
    \def\addlegendimage{\csname pgfplots@addlegendimage\endcsname}

\usepackage{subcaption}
\definecolor{pixel 0}{HTML}{FFFFFF}
\definecolor{pixel 1}{HTML}{FF0000} 

\usepackage[normalem]{ulem} 

\newcommand{\tw}[1]{\texttt{#1}}
\newcommand{\name}{\textsc{Poppi}}
\newcommand{\namem}{\textsc{Poppi$_{\textsc{MIL}}$}}
\newcommand{\popper}{\textsc{Popper}}
\newcommand{\metagol}{\textsc{Metagol}}
\newcommand{\inspire}{\textsc{Inspire}}
\newcommand{\citet}[1]{\citeauthor{#1} \shortcite{#1}}

\usepackage{amssymb,cancel}
\newcommand{\metagoln}{
\textsc{Metagol}$_{\scalebox{0.6}{\cancel{\rotatebox[origin=c]{245}{$\circlearrowleft$}}}}$
}

\theoremstyle{definition}
\newtheorem{definition}{Definition}
\newtheorem{example}{Example}
\newtheorem{theorem}{Theorem}

\newenvironment{code}
{
\ttfamily
\small
\begin{center}
\begin{tabular}{l}
}
{
\end{tabular}
\end{center}
\par
}

\title{Predicate Invention by Learning From Failures}

\author{
Andrew Cropper\And
Rolf Morel
\\
\affiliations
University of Oxford
\emails
{firstname.lastname}@cs.ox.ac.uk
}

\begin{document}

\maketitle
\begin{abstract}
Discovering novel high-level concepts is one of the most important steps needed for human-level AI.
In inductive logic programming (ILP), discovering novel-high level concepts is known as \emph{predicate invention} (PI).
Although seen as crucial since the founding of ILP, PI is notoriously difficult and most ILP systems do not support it.
In this paper, we introduce \name{}, an ILP system that formulates the PI problem as an answer set programming problem.
Our experiments show that (i) PI can drastically improve learning performance when useful, (ii) PI is not too costly when unnecessary, and (iii) \name{} can substantially outperform existing ILP systems.
\end{abstract}\section{Introduction}

Some of the greatest paradigm shifts in science come from the invention of new predicates, such as Galileo's invention of acceleration and Joule's invention of thermal energy \cite{ai:book}.
\citet{russell:humancomp} argues that developing techniques to discover novel high-level concepts is one of the most important steps needed to reach human-level AI.

In inductive logic programming (ILP) \cite{mugg:ilp}, a form of logic-based machine learning, discovering novel high-level concepts is known as \emph{predicate invention} (PI) \cite{stahl:pi}.
A classical example of PI is learning the definition of \emph{grandparent} given only the relations \emph{mother} and \emph{father}. An ILP system without PI can learn the hypothesis:

\begin{center}
\begin{tabular}{l}
\emph{grandparent(A,B) $\leftarrow$ mother(A,C),mother(C,B)}\\
\emph{grandparent(A,B) $\leftarrow$ mother(A,C),father(C,B)}\\
\emph{grandparent(A,B) $\leftarrow$ father(A,C),mother(C,B)}\\
\emph{grandparent(A,B) $\leftarrow$ father(A,C),father(C,B)}\\
\end{tabular}
\end{center}

\noindent
By contrast, a system with PI can learn the hypothesis:

\begin{center}
\begin{tabular}{l}
\emph{grandparent(A,B) $\leftarrow$ inv(A,C),inv(C,B)}\\
\emph{inv(A,B) $\leftarrow$ mother(A,B)}\\
\emph{inv(A,B) $\leftarrow$ father(A,B)}\\
\end{tabular}
\end{center}

\noindent
The symbol \emph{inv} is invented and can be understood as \emph{parent}.
Introducing this symbol makes the hypothesis smaller, both in the number of literals and clauses, which is beneficial because the search complexity of ILP is often a function of the size of the smallest solution \cite{playgol}.

Although seen as crucial since the earliest days of ILP \cite{cigol} and repeatedly stated as a major challenge \cite{pedro:pi,ilp20,kramer:ijcai20}, most ILP systems do not support PI \cite{ilp30}.
As \citet{kramer1995predicate} states, PI is difficult because it is unclear \emph{when} and \emph{how} we should invent a symbol, how many arguments should it have, what its types are, etc.

Interest in PI has resurged due to \metagol{} \cite{mugg:metagold,metaho}, which uses \emph{metarules}, second-order clauses, to reduce the complexity of PI by restricting the syntax of hypotheses.
The major limitation of \metagol{} and other metarule-based approaches \cite{wang:softpi,dilp,hexmil,celine:bottom} is the inherent need for metarules, which a user must provide.
Given insufficient metarules, these approaches cannot learn any solution, let alone one with PI \cite{reduce}.

To overcome this limitation, we introduce a PI approach based on \emph{learning from failures} (LFF) \cite{popper}, which frames the ILP problem as an answer set programming (ASP) problem \cite{asp}.
We build on LFF by framing the PI problem as an ASP problem.
Compared to existing approaches, our approach does not need a user to manually provide metarules or specify the arity and argument types of an invented symbol \cite{ilasp}, i.e. our approach supports \emph{automatic} PI.
Moreover, our approach retains the benefits of LFF, notably scalability and the ability to learn recursive programs.
Overall, our contributions are:
\begin{itemize}
    \item
    We extend LFF to support predicate invention.
    We introduce a new type of \emph{redundancy} constraint and show that it is sound with respect to optimal solutions.

    \item
    We introduce \name{}, a LFF implementation that supports automatic PI.
    \item
    We experimentally show that (i) PI can drastically improve learning performance when \emph{useful}, (ii) PI is not too costly when unnecessary, and (iii) \name{} can substantially outperform existing ILP systems.
\end{itemize}
\section{Related Work}
\label{sec:related}


Early work on PI was based on the idea of inverse resolution and \emph{W} operators \cite{cigol}.
Although these approaches could support PI, they never demonstrated completeness, partly because of the lack of a declarative bias to delimit the hypothesis space \cite{mugg:metagold}.

One PI approach is to use \emph{placeholders} \cite{DBLP:conf/ilp/LebanZB08} to predefine invented symbols.
However, this approach requires that a user manually specify the arity and argument types of an `invented' symbol \cite{ilasp}, which rather defeats the point, or requires enumerating every possible symbol \cite{apperception}.
By contrast, \name{} does not need a user to manually pre-define predicates and can automatically infer argument types, i.e. \name{} supports \emph{automatic PI}.
\inspire{} \cite{inspire} also uses ASP for PI, but only learns single-clause non-recursive programs.
By contrast, \name{} learns recursive multi-clause programs.
PI can be performed by learning programs in a multi-task setting \cite{metabias} or as a pre-/post-processing step \cite{seb:cluster,playgol,alps,knorf}.
By contrast, \name{} performs PI whilst learning.

\metagol{} uses metarules to restrict the syntax of hypotheses and drive PI.
For instance, the \emph{chain} metarule ($P(A,B) \leftarrow Q(A,C), R(C,B)$) allows \metagol{} to induce programs such as \emph{f(A,B) $\leftarrow$ tail(A,C),tail(C,B)}, which drops the first two elements from a list.
To induce longer clauses, \metagol{} chains clauses together using PI, such as the following program which drops the first four elements from a list:
\begin{center}
$
P = \left\{
\begin{array}{l}
	\emph{f(A,B)} \leftarrow \emph{inv(A,C),inv(C,B)}\\
	\emph{inv(A,B)} \leftarrow \emph{tail(A,C),tail(C,B)}\\
\end{array}
\right\}
$
\end{center}

\noindent
The major limitation of metarule-driven PI approaches \cite{wang:softpi,dilp,hexmil,celine:bottom} is their inherent need for user-supplied metarules, which, in some cases, are impossible to provide \cite{reduce}.
\name{} overcomes this major limitation because it does not need metarules.
Moreover, in Section \ref{sec:poppi-mil}, we show that \name{} subsumes \metagol{} because it can replicate its search strategy.

\name{} builds on the LFF system \popper{}, which works in three repeating stages: \emph{generate}, \emph{test}, and \emph{constrain}.
In the generate stage, \popper{} builds a logic program that satisfies a set of \emph{hypothesis constraints}.
In the test stage, \popper{} tests a program against the examples.
A hypothesis \emph{fails} when it is incomplete (does not entail all the positive examples) or inconsistent (entails a negative example).
If a hypothesis fails, \popper{} learns hypothesis constraints from the failure, which it then uses to restrict subsequent generate stages.
It repeats this process until it finds a complete and consistent program.
\name{} extends \popper{} to support automatic PI.
\section{Problem Setting}
\label{sec:setting}

We now define PI and extend LFF to support PI.
We assume familiarity with logic programming \cite{lloyd:book}.



\subsection{Predicate Invention}
We focus on the ILP learning from entailment setting \cite{luc:book}.

\begin{definition}[\textbf{ILP input}]
An ILP input is a tuple $(B,E^+,E^-,\mathcal{H})$, where $B$ is a logic program denoting background knowledge, $E^+$ and $E^-$ are sets of ground atoms which represent positive and negative examples respectively, and $\mathcal{H}$ is a hypothesis space (the set of all hypotheses).
\end{definition}
\noindent
In practice, the hypothesis space is restricted by a language bias, such as metarules or mode declarations.
The ILP problem is to find a \emph{solution}:
\begin{definition}[\textbf{Solution}]
Given an ILP input $(B,E^+,E^-,\mathcal{H})$, the ILP problem is to return a hypothesis $H \in \mathcal{H}$ such that
$\forall e \in E^+, H \cup B \models e$ and $\forall e \in E^-, H \cup B \not\models e$.
\end{definition}


\noindent
\citet{mugg:pi} and \citet{stahl:pi} both provide definitions for PI, which we adapt.
We denote the predicate signature (the set of all predicate symbols) of a logic program $P$ as \emph{ps(P)}.

\begin{definition}[\textbf{Predicate invention}]
Let $(B,E^+,E^-,\mathcal{H})$ be an ILP input and $H \in \mathcal{H}$.
Then $H$ contains an invented predicate symbol if and only if $ps(H) \setminus ps(B \cup E^+ \cup E^-) \neq \emptyset$.
\end{definition}

\noindent
PI is necessary when needed to learn a solution:

\begin{definition}[\textbf{Necessary predicate invention}]
Let $I = (B,E^+,E^-,\mathcal{H})$ be an ILP input.
Then PI is necessary for $I$ if no $H$ in $\mathcal{H}$ is a solution without PI.
\end{definition}

\noindent
PI is useful when it allows us to learn a better solution:

\begin{definition}[\textbf{Useful predicate invention}]
\label{def:pi}
Let $(B,E^+,E^-,\mathcal{H})$ be an ILP input and $cost : \mathcal{H} \mapsto R$ be an arbitrary cost function.
Then PI is useful when (i) there is a $H \in \mathcal{H}$ such that $H$ is a solution with invented predicate symbols, and (ii) $\forall H' \in \mathcal{H}$, where $H'$ is a solution, $cost(H) \leq cost(H')$.
\end{definition}

\noindent
In this paper, we define the \emph{cost(H)} to be the total number of literals in the logic program $H$.

\subsection{Learning From Failures}
We extend LFF to support PI.
LFF uses \emph{predicate declarations} to restrict what predicate symbols may appear in a hypothesis.
A predicate declaration is a ground atom of the form $\emph{head\_pred}(\emph{p},a)$ or $\emph{body\_pred}(\emph{p},a)$ where \emph{p} is a predicate symbol of arity $a$.
Given a set of predicate declarations $D$, a definite clause $C$ is \emph{declaration consistent} when (i) if $\emph{p}/m$ is the predicate in the head of $C$, then $\emph{head\_pred}(\emph{p},m)$ is in $D$, and (ii) for all $\emph{q}/n$ predicate symbols in the body of $C$, $\emph{body\_pred}(\emph{q},n)$ is in $D$.
LFF uses \emph{hypothesis constraints} to restrict the hypothesis space.
Let $\mathcal{L}$ be a language that defines hypotheses, i.e.~a meta-language.
Then a hypothesis constraint is a constraint expressed in $\mathcal{L}$.
Let $C$ be a set of hypothesis constraints written in a language $\mathcal{L}$.
A set of definite clauses $H$ is \emph{consistent} with $C$ if, when written in $\mathcal{L}$, $H$ does not violate any constraint in $C$.

We define the LFF problem:


\begin{definition}[\textbf{LFF input}]
\label{def:probin}
The \emph{LFF} input is a tuple $(E^+, E^-, B, D, C)$ where $E^+$ and $E^-$ are sets of ground atoms denoting positive and negative examples respectively; $B$ is a Horn program denoting background knowledge; $D$ is a set of predicate declarations; and $C$ is a set of hypothesis constraints.
\end{definition}

\noindent
A definite program is a \emph{hypothesis} when it is consistent with both $D$ and $C$.
We denote the set of such hypotheses as $\mathcal{H}_{D,C}$.
We define a LFF solution:

\begin{definition}[\textbf{LFF solution}]
\label{def:solution}
Given an input tuple $(E^+, E^-, B, D, C)$, a hypothesis $H \in \mathcal{H}_{D,C}$ is a \emph{solution} when $H$ is \emph{complete} ($\forall e \in E^+, \; B \cup H \models e$) and \emph{consistent} ($\forall e \in E^-, \; B \cup H \not\models e$).
\end{definition}

\noindent
If a hypothesis is not a solution then it is a \emph{failure} or a \emph{failed hypothesis}.
A hypothesis is \emph{incomplete} when $\exists e \in E^+, \; H \cup B \not \models e$.
A hypothesis is \emph{inconsistent} when $\exists e \in E^-, \; H \cup B \models e$.
A hypothesis is \emph{totally incomplete} when $\forall e \in E^+, \; H \cup B \not \models e$.
We define an \emph{optimal} solution:

\begin{definition}[\textbf{Optimal solution}]
\label{def:opthyp}
Given an input tuple $(E^+, E^-, B, D, C)$, a hypothesis $H \in \mathcal{H}_{D,C}$ is \emph{optimal} when (i) $H$ is a solution, and (ii) $\forall H' \in \mathcal{H}_{D,C}$, where $H'$ is a solution, $cost(H) \leq cost(H')$.
\end{definition}

\subsection{Redundancy Constraints}
\label{sec:picons}

The key idea of LFF is to learn hypothesis constraints from failed hypotheses.
\citet{popper} introduce constraints based on subsumption \cite{plotkin:thesis}.
A clause $C_1$ \emph{subsumes} a clause $C_2$ ($C_1 \preceq C_2$) if and only if there exists a substitution $\theta$ such that $C_1\theta \subseteq C_2$.
A clausal theory $T_1$ subsumes a clausal theory $T_2$ ($T_1 \preceq T_2$) if and only if $\forall C_2 \in T_2, \exists C_1 \in T_1$ such that $C_1$ subsumes $C_2$.
A clausal theory $T_1$ is a \emph{specialisation} of a clausal theory $T_2$ if and only if $T_2 \preceq T_1$.
A clausal theory $T_1$ is a \emph{generalisation} of a clausal theory $T_2$ if and only if $T_1 \preceq T_2$.

If a hypothesis $H$ is incomplete, a \emph{specialisation} constraint prunes specialisations of $H$, as they are also guaranteed to be incomplete.
Likewise, if a hypothesis $H$ is inconsistent, a \emph{generalisation} constraint prunes generalisations of $H$, as they are also guaranteed to be inconsistent.
\citet{popper} show that generalisation and specialisation constraints are \emph{sound} in that they never prune solutions from the hypothesis space.
The authors also introduce a third type of constraint called an \emph{elimination} constraint and show it is sound for optimal solutions.
If a hypothesis $H$ is totally incomplete, then there is no need for $H$ (or a specialisation of it) to appear in a complete and \emph{separable} hypothesis.
A separable hypothesis $H$ is one where no predicate symbol in the head of a clause in $H$ occurs in the body of a clause in $H$.
Elimination constraints are unsuitable for PI because an invented predicate should always appear in the body of a clause.

To overcome this limitation, we introduce \emph{redundancy} constraints.
We adapt the standard notion of a dependency graph \cite{apt:negationsurvey} of a program from predicates  to clauses:

\begin{definition}[\textbf{Dependency}]
The \emph{dependency graph} for the definite program $P$ is a directed graph where the nodes are the clauses in $P$.
There is an edge between nodes $C_1$ and $C_2$ if the head predicate symbol of $C_2$ occurs in the body of $C_1$.
We say $C_1$ \emph{depends} on $C_2$ if there is a path from $C_1$ to $C_2$.
\end{definition}

\noindent
We require that each clause in a hypothesis is reachable from a target predicate clause (a clause that generalises the examples).
We use the dependency relation to characterise when a program is at least as specific as another program:


\begin{definition}[\textbf{Contained specialisation}]
Let $P$ and $Q$ be definite programs and $P' = \{C \in Q \;|\; \exists D \in P, D \preceq C\}$.
A clause $C \in P'$ \emph{$P$-specialises} $Q$ if
(i) $C$ does not depend on any clause in $Q \setminus P'$, and
(ii) no clause in $Q \setminus P'$ depends on $C$.
\end{definition}


\noindent
The idea is that part of $Q$ cannot entail more than $P$, as the clauses of $Q$ on paths to $C$ specialise $P$, as we now illustrate:


\begin{example}
\label{ex:redundant_clause}
Let $P = \{ C_1, C_2 \}$, $C_1 = \emph{f(A,B)} \leftarrow \emph{inv(A,C),inv(C,B)}$, $C_2 = \emph{inv(A,B)} \leftarrow \emph{tail(A,C),tail(C,B)}$, $C_Q = \emph{f(A,B) $\leftarrow$ reverse(A,B)}$, and $Q = P \cup \{ C_Q \}$.
In $Q$, neither $C_2$ nor $C_Q$ depend on one another, so $C_2$ $P$-specialises $Q$.
As $C_2$ cannot help $Q$ entail more than $P$ does, if $P$ is a totally incomplete hypothesis then $C_2$ is \emph{redundant} in $Q$.
\end{example}

\noindent
If the hypothesis $P$ is totally incomplete and there is a clause that $P$-specialises the hypothesis $Q$, we call $Q$ a \emph{redundant} hypothesis.
We show that redundant hypotheses are not optimal:

\begin{theorem}[\textbf{Redundancy soundness}]
\label{prop:redundant_optimality}
Let $(E^+, E^-, B, D, C)$ be a LFF input, $H_1, H_2 \in \mathcal{H}_{D,C}$, and $H_1$ be totally incomplete.
If $H_2$ has a clause that $H_1$-specialises it, then $H_2$ is not an optimal solution.
\end{theorem}

\noindent
The proof of Theorem \ref{prop:redundant_optimality} is in Appendix A.
We call constraints that only prune redundant hypotheses \emph{redundancy} constraints.
We show that redundancy constraints prune more than elimination constraints:

\begin{theorem}[\textbf{Redundancy effectiveness}]
\label{prop:more_pruning}
Let $(E^+, E^-, B, D, C)$ be a LFF input and $H \in \mathcal{H}_{D,C}$ be totally incomplete.
As (i) all separable hypotheses containing specialisations of $H$ are redundant, and
(ii) there exist $H' \in \mathcal{H}_{D,C}$ that are redundant but not separable,
redundancy constraints prune strictly more than elimination constraints.
\end{theorem}

\begin{proof}
For (i), let $H^*$ be a separable hypothesis containing a specialisation of $H$.
As separability implies no clause of $H^*$ depends on another, any clause of $H^*$ specialising a clause of $H$ is redundant.
For (ii) Example \ref{ex:redundant_clause} suffices, taking $H = P$ and $H' = Q$.
Hypothesis $Q$ is redundant but not separable.
\end{proof}

\section{\name{}}
\label{sec:impl}

\name{} extends \popper{} with a new PI module and the ability to learn redundancy constraints from failures.
Due to space limitations, we cannot describe \popper{} in detail, so we refer the reader to the original paper for a detailed explanation \cite{popper}.
\name{} works in three repeating stages: \emph{generate}, \emph{test}, and \emph{constrain}, which we now describe.

\subsubsection{Generate}
\name{} takes as input (i) predicate declarations, (ii) hypothesis constraints, and (iii) bounds on the maximum number of variables, literals, and clauses in a hypothesis.
\name{} returns an answer set which represents a definite program, if one exists.
For instance, consider an answer set with the following head (\tw{h\_lit}) and body (\tw{b\_lit}) literals:

\begin{code}
h\_lit(0,last,2,(0,1)), b\_lit(0,tail,2,(0,2)),\\
b\_lit(0,empty,1,(2,)), b\_lit(0,head,2,(0,1)),\\
h\_lit(1,last,2,(0,1)), b\_lit(1,tail,2,(0,2)),\\
b\_lit(1,last,2,(2,1))
\end{code}

\noindent
The first argument of each literal is the clause index, the second is the predicate symbol, the third is the arity, and the fourth is the literal variables, where \tw{0} represents \tw{A}, \tw{1} represents \tw{B}, etc.
This answer set corresponds to the definite program:

\begin{code}
last(A,B):- tail(A,C), empty(C), head(A,B).\\
last(A,B):- tail(A,C), last(C,B).
\end{code}

\noindent
Figure \ref{fig:alan} shows the base ASP program to generate programs.
\name{} extends \popper{} by allowing invented symbols (\tw{invented(P,A)}) to appear in head and body literals.

\name{} uses ASP constraints to ensure that a program is declaration consistent and obeys hypothesis constraints.
For instance, to prune programs where the head literal appears in the body, \name{} enforces the constraint:
\begin{code}
:-
    h\_lit(C,P,A,Vs),
    b\_lit(C,P,A,Vs).
\end{code}

\noindent
By later adding learned hypothesis constraints, \name{} eliminate answer sets and thus prunes the hypothesis space.

\begin{figure}[ht]
\centering
\small
\begin{minipage}{1.02\linewidth}
\small
\begin{lstlisting}[frame=single]
head_p(P,A):- head_pred(P,A).
head_p(P,A):- invented(P,A).
body_p(P,A):- body_pred(P,A).
body_p(P,A):- invented(P,A).
clause(0..N-1):- max_clauses(N).
{h_lit(C,P,A,Vs)}:- head_p(P,A),vars(A,Vs),clause(C).
{b_lit(C,P,A,Vs)}:- body_p(P,A),vars(A,Vs),clause(C).
\end{lstlisting}
\end{minipage}
\caption{
Partial listing of the program generator in \name{}.
}
\label{fig:alan}
\end{figure}
\noindent

\paragraph{PI Module.}
\label{sec:pi}
\name{} builds on \popper{} with a PI module, partially shown in Figure \ref{fig:pi}.
Lines 1-3 show the ASP choice rules which allow \name{} to perform PI.
\name{} automatically adds $n-1$ choice rules to allow $n-1$ invented symbols, where $n$ is the maximum number of clauses allowed in a program.
This module contains many constraints specific to PI to prune redundant programs.
Lines 4-6 define an ordering over invented symbols, similar to \metagol{}, and line 9 enforces this ordering.
Lines 7-8 ensure that if an invented symbol is in the head of a clause then it must also appear in the body of a clause (and vice-versa).
Line 10 ensures that \name{} uses invented predicates in order, e.g.~to prevent the symbol \emph{inv2} from appearing in a program if \emph{inv1} does not already appear.
This module also includes code (omitted for brevity) to infer the argument types of an invented predicate by how it is used in the program and propagates the types through the program.

\begin{figure}[ht]
\centering
\small
\begin{minipage}{1\linewidth}
\small
\begin{lstlisting}[frame=single,numbers=left,stepnumber=1]
{invented(inv1,1..A)}:- max_arity(A).
{invented(inv2,1..A)}:- max_arity(A).
{invented(inv3,1..A)}:- max_arity(A).
lower(P,Q):- head_pred(P,_), invented(Q).
lower(inv1,inv2).
lower(inv2,inv3),
:- invented(P,A), not h_lit(_,P,A,_).
:- invented(P,A), not b_lit(_,P,A,_).
:- invented(P,_), lower(Q,P), not invented(Q,_).
:- h_lit(C,P,_,_), b_lit(C,Q,_,_), lower(Q,P).
\end{lstlisting}
\end{minipage}
\caption{
Partial listing of the PI module in \name{}.
}
\label{fig:pi}
\end{figure}
\noindent

\subsubsection{Test and constrain}
The test stage of \name{} is identical to \popper{}: \name{} transforms an answer set to a definite program (a Prolog program) and tests it against the training examples.
If a hypothesis fails, then, in the constrain stage, \name{} derives ASP constraints which it adds to the generator program to prune answer sets and constrain subsequent hypothesis generation.
For instance, suppose this hypothesis is inconsistent:

\begin{code}
last(A,B):- reverse(A,C), head(C,B).\\
\end{code}

\noindent
Then \name{} generates a \emph{generalisation} constraint to prune generalisations of this hypothesis:

\begin{lstlisting}[frame=single]
seen(C,id1):- h_lit(C,last,2,(V0,V1)),
  b_lit(C,reverse,2,(V0,V2)), b_lit(C,head,2,(V2,V1)).
:- seen(C,id1), clause_size(C,2).
\end{lstlisting}

\noindent
\name{} extends \popper{} with redundancy constraints (Section \ref{sec:picons}).
For instance, suppose this program is totally incomplete:

\begin{code}
f(A,B):- inv1(A,C), right(C,B).\\
inv1(A,B):- right(A,C), right(C,B).
\end{code}

\noindent
Whereas \popper{} would only generate a specialisation constraint, \name{} additionally generates a redundancy constraint:

\begin{lstlisting}[frame=single]
seen(C,c1):- h_lit(C,f,2,(V0,V1)),
  b_lit(C,inv1,2,(V0,V2)), b_lit(C,right,2,(V2,V1)).
seen(C,c2):- h_lit(C,inv1,2,(V0,V1)),
  b_lit(C,right,2,(V0,V2)), b_lit(C,right,2,(V2,V1)).
seen(p1):- seen(C0,c1), seen(C1,c2).
:- seen(p1), num_clauses(f,1), num_recursive(inv1,0).
:- seen(p1), num_clauses(inv1,1), num_recursive(f,0).
\end{lstlisting}

\noindent
Appendix B describes the algorithm \name{} uses to build redundancy constraints.

\paragraph{Loop.}
\name{} repeats the generate/test/constraint loop.
To find an optimal solution, \name{} increases the number of literals allowed in a program when the hypothesis space is empty at a certain program size (e.g. when there are no more programs to generate).
To improve efficiency, \name{} uses Clingo's multi-shot solving \cite{multishot} to maintain state between the three stages and thus remember learned conflicts.
This loop repeats until either (i) \name{} finds an optimal solution, or (ii) there are no more hypotheses to test.

\paragraph{\namem{}}
\label{sec:poppi-mil}
\metagol{} needs metarules to define the hypothesis space and drive PI.
\popper{} can simulate metarules through hypothesis constraints \cite{popper}.
By supporting PI, \name{} subsumes \metagol{}.
In our experiments, we directly compare \metagol{} against \name{} when given identical metarules, which we call \namem{}.
\section{Experiments}
\label{sec:exp}

Few papers empirically evaluate PI \cite{playgol,celine:bottom} and none in a systematic manner.
Therefore, our experiments try to answer the question:

\begin{description}
\item[Q1] How beneficial is PI when it is useful?
\end{description}

\noindent
To answer \textbf{Q1}, Experiments \ref{exp:robots} and \ref{exp:tail} compare ILP systems with and without PI on problems purposely designed to benefit from PI.
Specifically, we compare (i) \name{} against \popper{}, and (ii) \metagol{} against \metagol{} without the ability to reuse invented predicates, which we denote as \metagoln{}.

PI should improve performance when it is useful (Definition \ref{def:pi}), but what about when it is unnecessary?
In other words, can PI be harmful?
To answer this question, our experiments try to answer the question:

\begin{description}
\item[Q2] How costly is PI when it is unnecessary?
\end{description}

\noindent
To answer \textbf{Q2}, Experiment \ref{exp:lists} compares the same systems on a problem that should not benefit from PI.

To test our claim that \name{} can go beyond existing systems that support PI, our experiments try to answer the question:
\begin{description}
\item[Q3] How well does \name{} perform against other ILP systems?
\end{description}

\noindent
To answer \textbf{Q3}, we compare  \name{} against \popper{}.
We also compare \name{} against \metagol{}, which is the only system that supports automatic PI and learning recursive programs.



\paragraph{Settings.}
We use \name{} and \popper{} with the same settings as \citet{popper}: at most 6 variables in a clause, at most 5 body literals, and at most 3 clauses\footnote{
\popper{} cannot learn solutions in some of the experiments using these setting and quickly says that a problem is unsolvable, e.g. in Experiment \ref{exp:robots} \popper{} needs to learn a solution with 24 body literals.
However, when we allowed 24 body literals, \name{} could still not learn a solution in the given time.
We therefore use these settings but when \popper{} cannot find a solution, we assign it the maximum learning time.
}.
\metagol{} needs metarules.
We consider two sets of metarules: \emph{with} and \emph{without} recursion, listed in Appendix C.
\namem{} uses the same metarules encoded as hypothesis constraints.
We use a 3.8 GHz 8-Core Intel Core i7 with 32GB of ram.
Note that all the ILP systems only use a single CPU.


\subsection{Experiment 1 - Robot Planning}
\label{exp:robots}
This experiment aims to answer \textbf{Q1}.
We therefore use a problem where PI should be useful.

\paragraph{Materials.}
A robot is in a $100^2$ grid world.
An example is an atom $f(x,y)$ where $x$ and $y$ are initial and final states respectively.
A state describes the position of the robot.
The problem is to learn a plan to move right $k$ times.
As BK, we provide the single dyadic predicate \emph{right}, which moves the robot right one step.

\paragraph{Method.}
For each $k$ in $\{4,8,12,16,20,24\}$, we generate 10 positive training examples of the robot moving right $k$ times and $k-1$ negative training examples of the robot moving right $i$ times for $i\dots k-1$.
The target solution is a chain of $k$ \emph{right} actions.
As the value $k$ grows, PI should become more useful because a system can invent and reuse chains of actions.
We repeat the experiment five times.
We measure learning times and standard error.
We set a learning timeout of two minutes.

\paragraph{Results.}
\label{sec:robores}
Figure \ref{fig:robot-results} shows the results.
Without recursion, \metagol{} is the best performing system because it can chain invented predicate symbols, such as this program for when $k=16$, where each $f_i$ is invented:

\begin{code}
f(A,B):- f1(A,C),f1(C,B).\\
f1(A,B):- f2(A,C),f2(C,B).\\
f2(A,B):- f3(A,C),f3(C,B).\\
f3(A,B):- right(A,C),right(C,B).
\end{code}

\noindent
This reuse allows Metagol (and \namem{}) to move right $2^n$ times using only $n$ clauses, i.e. to move right $k>2$ times, \metagol{} requires at most $\ln(k)$ clauses.
By contrast, \metagoln{} performs poorly because it cannot reuse invented symbols and to move right $k>2$ times requires $k-1$ clauses.

\name{} substantially outperforms \popper{}.
To move right $k$ times, \popper{} must learn a program with a single clause formed of $k+1$ literals.
By contrast, \name{} can learn a more compact program, such as a program which requires only 10 literals to move right 16 times:

\begin{code}
f(A,B):- inv1(A,C),inv1(C,D),inv1(D,E),inv1(E,B).\\
inv1(A,B):- right(A,C),right(C,D),right(D,E),right(E,B).
\end{code}

\noindent
With recursion, \metagol{}  performs poorly because of its inefficient search algorithm.
By contrast, \name{} and \namem{} both perform well with recursion.

Overall, these results suggest that PI is extremely beneficial when useful (\textbf{Q1}) and that \name{} can drastically outperform existing systems (\textbf{Q3}).

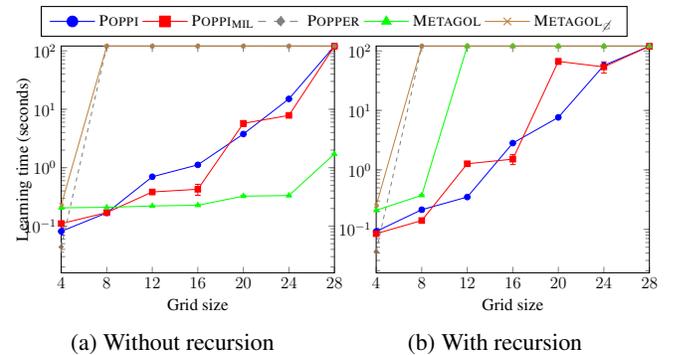
\begin{figure}[ht]
\centering
\begin{tikzpicture}
\begin{customlegend}[legend columns=5,legend style={nodes={scale=0.6, transform shape},align=left,column sep=0ex},
        legend entries={\name{}, \namem{}, \popper{}, \metagol{}, \metagoln{}}]
        \addlegendimage{mark=*,blue}
        \addlegendimage{mark=square*,red}
        \addlegendimage{mark=diamond*,gray,dashed}
        \addlegendimage{mark=triangle*,green}
        \addlegendimage{mark=x,brown}
        \end{customlegend}
\end{tikzpicture}
\begin{subfigure}{.5\linewidth}
\centering
\pgfplotsset{every tick label/.append style={font=\Large}}
\begin{tikzpicture}[scale=.53]
    \begin{axis}[
    xlabel=Grid size,
    ylabel=Learning time (seconds),
    xmin=4,xmax=28,
    ymin=0,ymax=121,
    ymode=log,
    xtick={4,8,12,...,28},
    ylabel style={yshift=-3mm},
    label style={font=\Large}
    ]

\addplot+[blue,mark=*,mark options={fill=blue},error bars/.cd,y dir=both,y explicit]
table [
x=size,
y=time,
y error plus expr=\thisrow{error},
y error minus expr=\thisrow{error},
] {data/robots/poppi-time.table};

\addplot+[red,mark=square*,mark options={color=red},error bars/.cd,y dir=both,y explicit]
table [
x=size,
y=time,
y error plus expr=\thisrow{error},
y error minus expr=\thisrow{error},
] {data/robots/poppi-mil-time.table};

\addplot+[gray,mark=diamond*,dashed,error bars/.cd,y dir=both,y explicit]
table [
x=size,
y=time,
y error plus expr=\thisrow{error},
y error minus expr=\thisrow{error},
] {data/robots/popper-time.table};

\addplot+[green,mark=triangle*,mark options={color=green},error bars/.cd,y dir=both,y explicit]
table [
x=size,
y=time,
y error plus expr=\thisrow{error},
y error minus expr=\thisrow{error},
] {data/robots/metagol-time.table};

\addplot+[brown,mark=x,mark options={color=brown},error bars/.cd,y dir=both,y explicit]
table [
x=size,
y=time,
y error plus expr=\thisrow{error},
y error minus expr=\thisrow{error},
] {data/robots/metagol-nopi-time.table};

    \end{axis}
  \end{tikzpicture}
\caption{Without recursion}
\end{subfigure}%
\begin{subfigure}{.5\linewidth}
\centering
\pgfplotsset{every tick label/.append style={font=\Large}}
\pgfplotsset{scaled x ticks=false}
\begin{tikzpicture}[scale=.53]
    \begin{axis}[
    xlabel=Grid size,
    xmin=4,xmax=28,
    ymin=0,ymax=121,
    xtick={4,8,12,...,28},
    ymode=log,
    label style={font=\Large}
    ]

\addplot+[blue,mark=*,mark options={fill=blue},error bars/.cd,y dir=both,y explicit]
table [
x=size,
y=time,
y error plus expr=\thisrow{error},
y error minus expr=\thisrow{error},
] {data/robots/poppi-rec-time.table};

\addplot+[red,mark=square*,mark options={color=red},error bars/.cd,y dir=both,y explicit]
table [
x=size,
y=time,
y error plus expr=\thisrow{error},
y error minus expr=\thisrow{error},
] {data/robots/poppi-mil-rec-time.table};

\addplot+[gray,mark=diamond*,dashed,error bars/.cd,y dir=both,y explicit]
table [
x=size,
y=time,
y error plus expr=\thisrow{error},
y error minus expr=\thisrow{error},
] {data/robots/popper-rec-time.table};

\addplot+[green,mark=triangle*,mark options={color=green},error bars/.cd,y dir=both,y explicit]
table [
x=size,
y=time,
y error plus expr=\thisrow{error},
y error minus expr=\thisrow{error},
] {data/robots/metagol-rec-time.table};

\addplot+[brown,mark=x,mark options={color=brown},error bars/.cd,y dir=both,y explicit]
table [
x=size,
y=time,
y error plus expr=\thisrow{error},
y error minus expr=\thisrow{error},
] {data/robots/metagol-nopi-rec-time.table};

    \end{axis}
  \end{tikzpicture}
\caption{With recursion}
\end{subfigure}
\caption{Robot experimental results.}
\label{fig:robot-results}
\end{figure}
\subsection{Experiment 2 - List Transformation}
\label{exp:tail}
This experiment aims to corroborate the results of Experiment 1 to answer \textbf{Q1}.
We therefore use a problem where PI should be useful.
This problem is modelled on the problems \citet{childhacker} use to model human intelligence.

\paragraph{Materials.}
The task is to learn a program to find the $kth$ element of a list.
An example is an atom $f(x,y)$, where $x$ is a list of characters and $y$ is a character.
A list is a random permutation of all the ASCII lowercase letters.
As BK, we provide the dyadic predicates \emph{head} and \emph{tail}.

\paragraph{Method.}
For each $k$ in $\{2,4,8,10,12,14\}$, we generate 1 positive example where $y$ is the $kth$ element of the list $x$.
We generate a negative training example for every other element in the list.
The target solution is a chain of $k-1$ \emph{tail} relations followed by a \emph{head} relation.
As the value $k$ grows, PI should become more useful because a system can invent and reuse chains of actions.
We repeat the experiment five times.
We measure learning times and standard error.
We set a learning timeout of two minutes.

\paragraph{Results.}
\label{sec:listres}
Figure \ref{fig:list-results} shows the results, which largely match those from Experiment \ref{exp:robots} and again suggest that PI is extremely beneficial when useful (\textbf{Q1}) and that \name{} can drastically outperform existing systems (\textbf{Q3}).

\begin{figure}[ht]
\centering
\begin{tikzpicture}
\begin{customlegend}[legend columns=5,legend style={nodes={scale=0.6, transform shape},align=left,column sep=0ex},
        legend entries={\name{}, \namem{}, \popper{}, \metagol{}, \metagoln{}}]
        \addlegendimage{mark=*,blue}
        \addlegendimage{mark=square*,red}
        \addlegendimage{mark=none,gray,dashed}
        \addlegendimage{mark=triangle*,green}
        \addlegendimage{mark=x,brown}
        \end{customlegend}
\end{tikzpicture}
\begin{subfigure}{.5\linewidth}
\centering
\pgfplotsset{every tick label/.append style={font=\Large}}
\begin{tikzpicture}[scale=.53]
    \begin{axis}[
    xlabel=Grid size,
    ylabel=Learning time (seconds),
    xmin=2,xmax=24,
    ymin=0,ymax=121,
    ymode=log,
    ylabel style={yshift=-3mm},
    xtick={4,8,12,...,24},
    label style={font=\Large}
    ]

\addplot+[blue,mark=*,mark options={fill=blue},error bars/.cd,y dir=both,y explicit]
table [
x=size,
y=time,
y error plus expr=\thisrow{error},
y error minus expr=\thisrow{error},
] {data/tail/poppi-time.table};

\addplot+[red,mark=square*,mark options={color=red},error bars/.cd,y dir=both,y explicit]
table [
x=size,
y=time,
y error plus expr=\thisrow{error},
y error minus expr=\thisrow{error},
] {data/tail/poppi-mil-time.table};

\addplot+[gray,mark=none,dashed,error bars/.cd,y dir=both,y explicit]
table [
x=size,
y=time,
y error plus expr=\thisrow{error},
y error minus expr=\thisrow{error},
] {data/tail/popper-time.table};

\addplot+[green,mark=triangle*,mark options={color=green},error bars/.cd,y dir=both,y explicit]
table [
x=size,
y=time,
y error plus expr=\thisrow{error},
y error minus expr=\thisrow{error},
] {data/tail/metagol-time.table};

\addplot+[brown,mark=x,mark options={color=brown},error bars/.cd,y dir=both,y explicit]
table [
x=size,
y=time,
y error plus expr=\thisrow{error},
y error minus expr=\thisrow{error},
] {data/tail/metagol-nopi-time.table};

    \end{axis}
  \end{tikzpicture}
\caption{Without recursion}
\end{subfigure}%
\begin{subfigure}{.5\linewidth}
\centering
\pgfplotsset{every tick label/.append style={font=\Large}}
\pgfplotsset{scaled x ticks=false}
\begin{tikzpicture}[scale=.53]
    \begin{axis}[
    xlabel=Grid size,
    xmin=2,xmax=24,
    ymin=0,ymax=121,
    ymode=log,
    xtick={4,8,12,...,24},
    label style={font=\Large}
    ]

\addplot+[blue,mark=*,mark options={fill=blue},error bars/.cd,y dir=both,y explicit]
table [
x=size,
y=time,
y error plus expr=\thisrow{error},
y error minus expr=\thisrow{error},
] {data/tail/poppi-rec-time.table};

\addplot+[red,mark=square*,mark options={color=red},error bars/.cd,y dir=both,y explicit]
table [
x=size,
y=time,
y error plus expr=\thisrow{error},
y error minus expr=\thisrow{error},
] {data/tail/poppi-mil-rec-time.table};

\addplot+[gray,mark=none,dashed,error bars/.cd,y dir=both,y explicit]
table [
x=size,
y=time,
y error plus expr=\thisrow{error},
y error minus expr=\thisrow{error},
] {data/tail/popper-rec-time.table};

\addplot+[green,mark=triangle*,mark options={color=green},error bars/.cd,y dir=both,y explicit]
table [
x=size,
y=time,
y error plus expr=\thisrow{error},
y error minus expr=\thisrow{error},
] {data/tail/metagol-rec-time.table};

\addplot+[brown,mark=x,mark options={color=brown},error bars/.cd,y dir=both,y explicit]
table [
x=size,
y=time,
y error plus expr=\thisrow{error},
y error minus expr=\thisrow{error},
] {data/tail/metagol-nopi-rec-time.table};

    \end{axis}
  \end{tikzpicture}
\caption{With recursion}
\end{subfigure}
\caption{List transformation experimental results.}
\label{fig:list-results}
\end{figure}
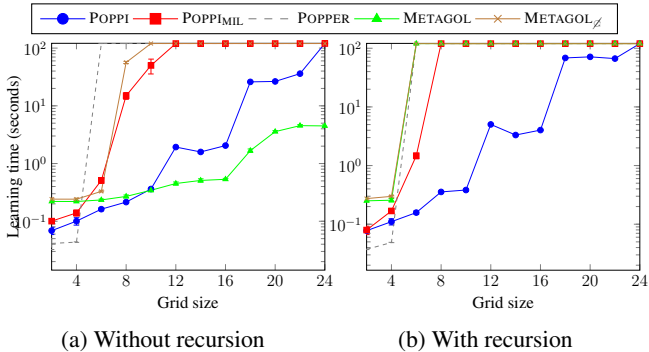
\subsection{Experiment 3 - Programming Puzzles}
\label{exp:lists}

Experiments 1 and 2 clearly show that PI can drastically improve performance when it is useful (\textbf{Q1}).
However, as we discussed at the start of Section \ref{sec:exp}, PI is not always useful.
The main purpose of this experiment is to answer \textbf{Q2}.
We therefore use problems where PI should \textbf{not} be useful.

\paragraph{Materials.}
We repeat the programming puzzles experiment from \cite{popper}.
The first column in Table \ref{tab:listaccs} shows the problems.
Note that these problems are extremely difficult for ILP systems and \popper{} is the first system that can reliably learn solutions to them.
We describe all the experimental materials in Appendix D.

\paragraph{Method.}
We generate 10 positive and 10 negative examples per problem.
Each example is randomly generated from lists up to length 50 formed of integers from 1-100.
We test on 1000 positive and 1000 negative randomly sampled examples.
We measure predictive accuracy and learning time.
We enforce a timeout of three minutes per task.
We repeat the experiment five times and measure the mean and standard error.

\paragraph{Results.}
Table \ref{tab:listaccs} shows that mean predictive accuracies.
\popper{} outperforms \name{} and a McNemar's test confirms the significance at the $p < 0.01$ level.
Because it supports PI, \name{} searches a larger hypothesis space than \popper{}, which is why it takes longer on some tasks.
When the predictive accuracy of \name{} is not 100\%, it is because it did not learn a solution in the given time in some of the trials.
Table \ref{tab:listtimes} shows this difference in learning time.
As is clear, \popper{} outperforms \name{} in all cases and a paired t-test confirms the significance at the $p < 0.01$ level, although the difference is always less than double.
Overall, these results suggest that PI is not too costly when not useful (\textbf{Q2}).
Moreover, as Table \ref{tab:listaccs} shows, \name{} drastically outperforms \metagol{} on all but three problems, where they both perform identically and helps answer (\textbf{Q3}).

\begin{table}[ht]
\centering
\small
\begin{tabular}{l|c|c|c|c}
\toprule
\textbf{Name} & \textbf{\name{}} & \textbf{\namem{}} & \textbf{\popper{}} & \textbf{\metagol{}}\\
\midrule
\tw{addhead} & \textbf{100} $\pm$ 0 & 50 $\pm$ 0 & \textbf{100} $\pm$ 0 & 50 $\pm$ 0 \\
\tw{dropk} & \textbf{100} $\pm$ 0 & 50 $\pm$ 0 & \textbf{100} $\pm$ 0 & 50 $\pm$ 0 \\
\tw{droplast} & \textbf{100} $\pm$ 0 & 50 $\pm$ 0 & \textbf{100} $\pm$ 0 & 50 $\pm$ 0 \\
\tw{evens} & \textbf{100} $\pm$ 0 & 50 $\pm$ 0 & \textbf{100} $\pm$ 0 & 50 $\pm$ 0 \\
\tw{finddup} & 98 $\pm$ 0 & 80 $\pm$ 12 & \textbf{99} $\pm$ 0 & 50 $\pm$ 0 \\
\tw{last} & \textbf{100} $\pm$ 0 & 50 $\pm$ 0 & \textbf{100} $\pm$ 0 & \textbf{100} $\pm$ 0 \\
\tw{len} & \textbf{100} $\pm$ 0 & 50 $\pm$ 0 & \textbf{100} $\pm$ 0 & 50 $\pm$ 0 \\
\tw{member} & \textbf{100} $\pm$ 0 & \textbf{100} $\pm$ 0 & \textbf{100} $\pm$ 0 & \textbf{100} $\pm$ 0 \\
\tw{sorted} & \textbf{100} $\pm$ 0 & 50 $\pm$ 0 & \textbf{100} $\pm$ 0 & 50 $\pm$ 0 \\
\tw{threesame} & \textbf{99} $\pm$ 0 & 50 $\pm$ 0 & \textbf{99} $\pm$ 0 & \textbf{99} $\pm$ 0 \\
\bottomrule
\end{tabular}
\caption{
Programming puzzles predictive accuracies.
We round accuracies to integer values.
The error is standard error.
}
\label{tab:listaccs}
\end{table}

\begin{table}[ht]
\centering
\small
\begin{tabular}{l|c|c|c|c}
\toprule
\textbf{Name} & \textbf{\name{}} & \textbf{\namem{}} & \textbf{\popper{}} & \textbf{\metagol{}}\\
\midrule
\tw{addhead} & 1 $\pm$ 0 & 120 $\pm$ 0.1 & \textbf{0.4} $\pm$ 0 & 120 $\pm$ 0 \\
\tw{dropk} & 1 $\pm$ 0.1 & \textbf{0.1} $\pm$ 0 & 1 $\pm$ 0 & 0.2 $\pm$ 0 \\
\tw{droplast} & 5 $\pm$ 1 & 120 $\pm$ 0.1 & \textbf{4} $\pm$ 0.8 & 120 $\pm$ 0 \\
\tw{evens} & 8 $\pm$ 0.5 & \textbf{0.7} $\pm$ 0 & 6 $\pm$ 0.2 & 120 $\pm$ 0 \\
\tw{finddup} & 91 $\pm$ 5 & 69 $\pm$ 22 & \textbf{48} $\pm$ 2 & 120 $\pm$ 0 \\
\tw{last} & 4 $\pm$ 0.2 & 120 $\pm$ 0.1 & 3 $\pm$ 0.3 & \textbf{0.6} $\pm$ 0.3 \\
\tw{len} & 25 $\pm$ 1 & 120 $\pm$ 0.1 & \textbf{16} $\pm$ 0.6 & 120 $\pm$ 0 \\
\tw{member} & 0.7 $\pm$ 0.1 & \textbf{0.2} $\pm$ 0 & \textbf{0.2} $\pm$ 0 & 0.3 $\pm$ 0 \\
\tw{sorted} & 32 $\pm$ 3 & \textbf{0.5} $\pm$ 0 & 35 $\pm$ 5 & 120 $\pm$ 0 \\
\tw{threesame} & \textbf{0.2} $\pm$ 0 & 0.6 $\pm$ 0 & \textbf{0.2} $\pm$ 0 & 0.6 $\pm$ 0.2 \\
\bottomrule
\end{tabular}
\caption{
Programming puzzles learning times.
We round times over 1 second to the nearest second.
The error is standard error.
}
\label{tab:listtimes}
\end{table}
\section{Conclusions and Limitations}

Although seen as key to human-level AI \cite{russell:humancomp} and repeatedly stated as a major challenge, most ILP systems do not support PI, and those that do have severe limitations (Section \ref{sec:related}).
To overcome these limitations, we have extended LFF with redundancy constraints which are sound (Theorem \ref{prop:redundant_optimality}) and more effective (Theorem \ref{prop:more_pruning}) than existing constraints.
We implemented our approach in the ILP system \name{}, which supports automatic PI, i.e. does not need metarules nor requires a user manually predefine invented symbols.
We have experimentally shown that (i) PI can drastically improve learning performance when \emph{useful}, (ii) PI is not too costly when \emph{unnecessary}, and (iii) \name{} can substantially outperform existing ILP systems.

There are limitations for future work to address.

\paragraph{Inefficient PI.}
Although \name{} convincingly outperforms \metagol{}, its PI technique is inefficient.
For instance, in Experiment \ref{exp:robots}, when learning to move right eight times, \name{} considers this hypothesis:

\begin{code}
f(A,B) :- inv1(A,C),right(C,D),right(D,B).\\
inv1(A,B) :- right(A,C),right(C,B).
\end{code}

\noindent
This hypothesis fails, as it only moves right four times.
From this failure \name{} learns constraints to prune specialisations of it.
However, \name{} will still consider logically equivalent hypotheses, such as:

\begin{code}
f(A,B) :- right(A,C),inv1(C,D),right(D,B).\\
inv1(A,B) :- right(A,C),right(C,B).
\end{code}

\noindent
We could address this issue by reasoning about unfolded hypotheses \cite{unfolding}.
In general, we expect to make further substantial efficiency improvements in \name{}.





\paragraph{Higher-arity invention.}
\metagol{} can only learn monadic and dyadic programs because of the restrictions imposed by metarules.
Therefore, to perform a fair experimental comparison, our experiments only consider inventing monadic and dyadic predicates.
\name{} can, however, invent predicates with arity greater than two.
In future work, we want to apply \name{} to more general problems with higher-arity invention.



\bibliographystyle{named}
\bibliography{manuscript}




\section{Appendix: Redundancy and Optimality}
\label{app:redundant_opt}

\begingroup
\renewcommand\theproposition{\ref{prop:redundant_optimality}}
\begin{theorem}[\textbf{Redundancy soundness}]
Let $(E^+, E^-, B, D, C)$ be a LFF input, $H_1, H_2 \in \mathcal{H}_{D,C}$, and $H_1$ be totally incomplete.
If $H_2$ has a clause that $H_1$-specialises it, then $H_2$ is not an optimal solution.
\end{theorem}
\endgroup

\begin{proof}
Suppose $H_2$ is a solution and that its clause $C$ $H_1$-specialises $H_2$.
Let $H_2'$ be the clauses of $H_2$ \emph{not} specialising a clause of $H_1$.
No clause of $H_2'$ depends on $C$ or has $C$ depend on it, meaning $C$ is only used by totally incomplete $H_2 \setminus H_2' \succeq H_1$.
Hence $C$ cannot have contributed to $H_2$ entailing positive examples.
Therefore $H_2 \setminus \{ C \}$ entails all positive examples and, due to the programs being definite, does not entail additional negative examples.
Conclude that $H_2$ is not optimal as a smaller solution exists.
\end{proof}

\section{Appendix: Building Redundancy Constraints}
\label{app:build_redundancy_constraints}

We describe how redundancy constraints are derived.
Let $H$ be a totally incomplete hypothesis.
For each clause in $H$, we generate a \emph{seen clause} rule.
For example, we assign the clause \texttt{f(A,B) :- inv1(A,C), right(C,B)} the identifier \texttt{c1} and we generate the rule:

\begin{lstlisting}[frame=single]
seen(C,c1):- head_lit(C,f,2,(V0,V1)),
  body_lit(C,inv1,2,(V0,V2)),
  body_lit(C,right,2,(V2,V1)).
\end{lstlisting}

\noindent
We assign each hypothesis an identifier, e.g.~\texttt{h1}.
Suppose that $H$ has \texttt{N} clauses.
Then the head of the following \emph{seen program} rule is true precisely when the program contains a specialisation of $H$:
\begin{lstlisting}[frame=single]
seen(h1):- seen(_,c1),...,seen(_,cN).
\end{lstlisting}

\noindent
We then consider all predicates defined by $H$.
When a predicate \texttt{p} depends on itself (is recursive), we say \texttt{p} is \emph{recursively called}.
For each predicate \texttt{p} of $H$ that is \emph{not} recursively called, we generate a constraint.
For fixed \texttt{p}, let \texttt{q1}, \ldots, \texttt{qM} be all the predicates of $H$ except \texttt{p}.
Let \texttt{\#cls($P$)} denote the number of clauses defining predicate $P$ in $H$.
\texttt{\#rec\_cls($P$)} denotes the number of recursive clauses of $P$.
Each non-recursively called predicate of $H$ then causes a redundancy constraint of the following shape to be generated:

\begin{lstlisting}[frame=single]
:- seen(h1),
   num_clauses(q1,#cls(q1)),
   ...,
   num_clauses(qM,#cls(qM)),
   num_recursive(p,#rec_cls(p)).
\end{lstlisting}

\section{Appendix: Metarules}
\label{sec:metarules}

\subsection{Experiments \ref{exp:robots} and \ref{exp:tail}}
\begin{lstlisting}[frame=single,caption={Metarules without recursion.}]
metarule([P,Q], [P,A,B], [[Q,A,B]]).
metarule([P,Q,R], [P,A,B], [[Q,A,B],[R,A]]).
metarule([P,Q,R], [P,A,B], [[Q,A,B],[R,B]]).
metarule([P,Q,R], [P,A,B], [[Q,A,C],[R,C,B]]).
\end{lstlisting}

\begin{lstlisting}[frame=single,caption={Metarules with recursion.}]
metarule([P,Q], [P,A,B], [[Q,A,B]]).
metarule([P,Q,R], [P,A,B], [[Q,A,B],[R,A]]).
metarule([P,Q,R], [P,A,B], [[Q,A,B],[R,B]]).
metarule([P,Q,R], [P,A,B], [[Q,A,C],[R,C,B]]).
metarule([P,Q], [P,A,B], [[Q,A,C],[P,C,B]]).
\end{lstlisting}

\subsection{Experiment \ref{exp:lists}}

\begin{lstlisting}[frame=single, caption={Metarules used in Experiment \ref{exp:lists}}]
metarule([P,Q], [P,A], [[Q,A]]).
metarule([P,Q], [P,A], [[Q,A]]).
metarule([P,Q,R], [P,A], [[Q,A,B],[R,B]]).
metarule([P,Q], [P,A], [[Q,A,B],[P,B]]).
metarule([P,Q,R], [P,A], [[Q,A,B],[R,A,B]]).
metarule([P,Q], [P,A,B], [[Q,A,B]]).
metarule([P,Q,R], [P,A,B], [[Q,A,B],[R,A,B]]).
metarule([P,Q,R], [P,A,B], [[Q,A],[R,A,B]]).
metarule([P,Q,R], [P,A,B], [[Q,A,B],[R,B]]).
metarule([P,Q,R], [P,A,B], [[Q,A,C],[R,C,B]]).
metarule([P,Q], [P,A,B], [[Q,A,C],[P,C,B]]).
metarule([P,Q], [P,A,B], [[Q,A,B]]).
\end{lstlisting}

\section{Appendix: Programming Puzzle Experimental Details}
\label{app:lists}
We give each system the following dyadic relations \emph{head}, \emph{tail}, \emph{decrement}, \emph{geq} and the monadic relations \emph{empty}, \emph{zero}, \emph{one}, \emph{even}, and \emph{odd}.
We also include the dyadic relation \emph{increment} in the \tw{len} experiment.
We had to remove this relation from the BK for the other experiments because when given this relation \metagol{} runs into infinite recursion on almost every problem and could not find any solutions.
We also include \emph{member/2} in the find duplicate problem.
We also include \emph{cons/3} in the \tw{addhead}, \tw{dropk}, and \tw{droplast} experiments.
We exclude this relation from the other experiments because \metagol{} does not easily support triadic relations.

\paragraph{\name{} settings.}
We set \name{} to use at most five unique variables, at most five body literals, and at most two clauses.
For each BK relation, we also provide both systems with simple types and argument directions (whether input or output).
Because \name{} can generate non-terminating Prolog programs, we set both systems to use a testing timeout of 0.1 seconds per example.
If a program times out, we view it as a failure.


\end{document}